\documentclass{article}
\usepackage[utf8]{inputenc}
\usepackage[margin=1in]{geometry}
\usepackage{amsmath, amsfonts, amssymb, amsthm}
\usepackage{hyperref}
\usepackage{xcolor}
\usepackage{algorithm}
\usepackage{algorithmic}
\usepackage{booktabs}
\usepackage{graphicx}
\usepackage{bbm}
\usepackage{authblk}

\newtheorem{theorem}{Theorem}
\newtheorem{lemma}[theorem]{Lemma}
\newtheorem{corollary}[theorem]{Corollary}

\newcommand{\blockcomment}[1]{}

\DeclareMathOperator*{\argmin}{argmin}
\newcommand{\nb}{\nabla}
\newcommand{\T}{\top}
\newcommand{\iid}{\stackrel{\tiny{\mathrm{iid}}}{\sim}}
\newcommand{\Bern}{\mathrm{Bernoulli}}

\newcommand{\E}{\mathbb{E}}
\newcommand{\R}{\mathbb{R}}
\renewcommand{\P}{\mathbb{P}}
\newcommand{\I}{\mathbbm{1}}

\newcommand{\D}{\mathcal{D}}
\renewcommand{\L}{\mathcal{L}}
\newcommand{\N}{\mathcal{N}}
\renewcommand{\O}{\mathcal{O}}

\newcommand{\Z}{\mathcal{Z}}

\newcommand{\hf}{\hat{f}}
\newcommand{\hft}{\hat{f_t}}
\newcommand{\hfp}{\widehat{f'}}
\newcommand{\hfpt}{\widehat{f'_t}}

\newcommand{\hnb}{\hat{\nb}}

\renewcommand{\b}{\beta}
\renewcommand{\d}{\delta}
\newcommand{\Th}{\Theta}
\renewcommand{\th}{\theta}
\newcommand{\e}{\varepsilon}
\newcommand{\g}{\gamma}
\newcommand{\n}{\eta}
\newcommand{\s}{\sigma}

\newcommand{\opt}{\th_{\mathrm{OPT}}}
\newcommand{\stab}{\th_{\mathrm{STAB}}}

\newcommand{\rrm}{\th_{\mathrm{RRM}}}

\newcommand{\lip}{L_{\mathrm{Lip}}}

\newcommand{\lmax}{\ell_{\mathrm{max}}}

\title{How to Learn when Data Reacts to Your Model: Performative Gradient Descent}
\author[1]{Zachary Izzo}
\author[1, 2]{Lexing Ying}
\author[3]{James Zou}
\affil[1]{Department of Mathematics, Stanford University}
\affil[2]{Institute for Computational and Mathematical Engineering, Stanford University}
\affil[3]{Department of Biomedical Data Science, Stanford University}
\date{}

\begin{document}

\maketitle

%A typical assumption in machine learning (ML) is that the training and test data are drawn iid from a fixed distribution. In practice, however, this is often not the case, and this problem has been studied extensively in the literature under the name \emph{dataset shift}. 
%Recently, \cite{Perdomo2020} introduced the concept of \emph{performative} dataset shift, where the the deployed model determines the change in data distribution.
%Works in this area have focused on finding a performatively stable point, which are model parameters which minimize the loss on the distribution induced by the model. In general, however, by choosing model parameters which induce a more favorable distribution, we may be able to reduce the loss even further. Finding an optimal set of model parameters which accounts for both the loss on the induced distribution, \emph{and} the favorability of the induced distribution itself, is the true goal of performative model training. In general, these optimal parameters may be far from a performatively stable point.
%a performatively stable point is not a performatively optimal point (i.e. one which minimizes the performative loss), and in certain situations a stable point may be very far from optimal.
%\james{need to more clearly explain the difference between optimal and stable here to a general reader has not read this literature}
\begin{abstract}
Performative distribution shift captures the setting where the choice of which ML model is deployed  changes the data distribution. For example, a bank which uses the number of open credit lines to determine a customer's risk of default on a loan may induce customers to open more credit lines in order to improve their chances of being approved. Because of the interactions between the model and data distribution, finding the optimal model parameters is challenging. Works in this area have focused on finding stable points, which can be far from optimal. Here we introduce \emph{performative gradient descent} (PerfGD), which is the first algorithm which provably converges to the performatively optimal point. PerfGD explicitly captures how changes in the model affects the data distribution and is simple to use. We support our findings with theory and experiments.   
\end{abstract}

\section{Introduction}
A common paradigm in machine learning is to assume access to training and test datasets which are drawn independently from a fixed distribution. In practice, however, this is frequently not the case, and changes in the underlying data distribution can lead to suboptimal model performance. This problem is referred to as \emph{distribution shift} or \emph{dataset shift}.

While there is an extensive body of literature on distribution shift \cite{Dataset2009}, most prior works have focused on exogenous changes in the data distribution due to e.g. temporal or spatial changes. For instance, such changes may occur when a model trained on medical imaging data from one hospital is deployed at a different hospital due to the difference in imaging devices. Time series analysis is another plentiful source of these types of dataset shift. A model trained on stock market data from 50 years ago is unlikely to perform well in the modern market due to changing economic trends; similarly, a weather forecasting model trained on old data will likely have poor performance without accounting for macroscopic changes in climate patterns.

More recently, researchers have sought to address \emph{endogenous} sources of distribution shift, i.e. where the change in distribution is induced by the choice of model. This setting, first explored in \cite{Perdomo2020}, is known as \emph{performative} distribution shift. Such effects can arise for a variety of reasons. The modeled population may try to ``game the system," causing individuals to modify some of their features to receive a more favorable classification (e.g. opening more credit lines to improve one's likelihood of being approved for a loan). Performative effects may also arise when viewing model output as a treatment. For instance, if a bank predicts a customer's default risk is high, the bank may assign that customer a higher interest rate, thereby increasing the customer's chance of defaulting \cite{Drusvyatskiy2020}. As ML systems play an ever-increasing role in daily life, accounting for performative effects will naturally become more and more critical for both the development of effective models and understanding the societal impact of ML.

The original paper \cite{Perdomo2020} and much of the follow-up research \cite{Mendler-Dunner2020, Drusvyatskiy2020, Brown2020} has viewed the performative setting as a dynamical system. The modeler repeatedly observes (samples from) the distribution arising from her choice of model parameters, then, treating this induced distribution as fixed, updates her model by reducing its loss on that fixed distribution. The primary question addressed by these works is under what conditions this process stabilizes, i.e. when will this process converge to a model which is optimal for the distribution it induces? A model with this property is known as a \emph{performatively stable point}.

While performatively stable points may be interesting from a theoretical standpoint, focusing on this objective misses the primary objective of model training: namely, obtaining the minimum \emph{performative loss}, i.e. the loss of the deployed model on the distribution it induces. The aforementioned previous works show that, in certain settings, a performatively stable point is a good proxy for a performatively optimal point, by bounding the distance between these two points in parameter space. In general, however, a performatively stable point may be far from optimal. In other less restrictive settings, a stable point may not even exist, and algorithms designed to find such a point may oscillate or diverge.

\subsection{Our contributions}
Motivated by these shortcomings, we introduce a new algorithm dubbed \emph{performative gradient descent} (PerfGD) which provably converges to the performatively optimal point under realistic assumptions on the data generating process. We demonstrate, both theoretically and empirically, the advantages of PerfGD over existing algorithms designed for the performative setting.

\subsection{Related work}
Dataset shift is not a new topic in ML, but earlier works focused primarily on exogenous changes to the data generating distribution. For a comprehensive survey, see \cite{Dataset2009}.

Performativity in machine learning was first introduced by \cite{Perdomo2020}. The authors introduced two algorithms (repeated risk minimization and repeated gradient descent) as methods for finding a performatively stable point, and showed that under certain smoothness assumptions on the loss and the distribution map, a performatively stable point must lie in a small neighborhood of the performatively optimal point. Their results relied on access to a large-batch or population gradient oracle. In the follow up work \cite{Mendler-Dunner2020}, the authors showed similar results for the stochastic optimization setting. The authors in \cite{Drusvyatskiy2020} analyze a general class of stochastic optimization methods for finding a performatively stable point. They view these algorithms as performing biased stochastic optimization on the fixed distribution introduced by the performatively stable point, and show that the bias decreases to zero as training proceeds. In \cite{Brown2020}, the authors give results analogous to those in \cite{Perdomo2020} when the distribution map also depends on the previous distribution. This models situations in which the population adapts to the model parameters slowly. In this case and under certain regularity conditions, RRM still converges to a stable point, and a stable point must lie within a small neighborhood of the optimum. We note that all of these works aim at finding a performatively stable, rather than performatively optimal, point.

Performativity in ML is closely related to the concept of strategic classification \cite{Hardt2016, Cai2015, Shavit2020, Kleinberg2019, Khajehnejad2019}. Strategic classification is a specific mechanism by which a population adapts to a choice of model parameters; namely, each member of the population alters their features by optimizing a utility function minus a cost. Performativity includes strategic classification as a special case, as we make no assumptions on the specific mechanism by which the distribution changes.

To the best of our knowledge, the only other work which computes the performatively optimal point is \cite{Munro2020}. However, this work differs from ours in several important ways. First, in \cite{Munro2020}, the planner may deploy a different model on each individual from the sample at each time step. In our setting, as in \cite{Perdomo2020}, the model deployment must be uniform across all agents in each time step; testing different models constitutes different deployments, and we also seek the optimal uniform model. Second, \cite{Munro2020} assumes that the performative shift results from strategic classification on the part of the agents. We trade these assumptions for parametric assumptions on the data generating process, but allow for a more general change in the data distribution (i.e. the change need not arise from a utility maximization problem.) In short, while superficially similar, our papers address unique settings and the results are in fact complementary.

Finally, training under performative distribution shift can be seen as a special instance of a zeroth-order optimization problem \cite{Duchi2015, Lattimore2020}, and our use of finite differences to approximate a gradient is a technique also employed by these works. However, the additional structure of our problem leads to algorithms better suited for the particular case of performative distribution shift.

The rest of the paper is structured as follows. In Section \ref{sec: notation}, we introduce the problem framework as well as notation that we will use throughout the paper. We also discuss previous algorithms for performative ML and explore their shortcomings. In Section \ref{sec: parametric}, we introduce our algorithm, performative gradient descent (PerfGD). In Section \ref{sec: theory}, we prove quantitative results on the accuracy and convergence of PerfGD. Section \ref{sec: applications} considers several specific applications of our method and verifies its performance empirically. We conclude in Section \ref{sec: conclusion} and introduce possible directions for future work.

%%%%%%%%%%%%%%%%%%%%%%%%%%%%%%%%%%%%%%%%%%%%%%%%%%%%%%%%%%%%%%%%%%%%%%%%%%%%%%%%%%%%%%%%%%%%%%%%%%%%%%%%%%%%%%%%%%%%%%%%%%%%%%%%%%%%%%%%%%%%%%%%%%%%%%%%%%%%%%%%%%%%%%%%
%%%%%%%%%%%%%%%%%%%%%%%%%%%%%%%%%%%%%%%%%%%%%%%%%%%%%%%%%%%%%%%%%%%%%%%%%%%%%%%%%%%%%%%%%%%%%%%%%%%%%%%%%%%%%%%%%%%%%%%%%%%%%%%%%%%%%%%%%%%%%%%%%%%%%%%%%%%%%%%%%%%%%%%%

\section{Setup and notation} \label{sec: notation}
We introduce notation which will be used throughout the rest of the paper.
\begin{itemize}
    \item $\Th \subseteq \R^p$ denotes the space of model parameters, which we assume is closed and convex.
    \item $\Z \subseteq \R^d$ denotes the sample space of our data.
    %\item $\X \subseteq \R^d$ denotes the feature space of the data. For both linear and logistic regression, we will have $d = p$.
    %\item $\Y$ denotes the label space of the data. For regression problems, we will have $\Y \subseteq \R$; for (binary) classification problems, we will have $\Y = \{0, 1\}$.
    %\item $\Z = \X \times \Y$ denotes the sample space.
    \item $\D:\Th \rightarrow \mathcal{P}(\Z)$ denotes the performative distribution map. That is, when we deploy a model with parameters $\th$, we receive data drawn iid from $\D(\th)$. We will assume that $\D$ is unknown; we only observe it indirectly from the data.
    %\item $\hD(\th)$ denotes the empirical distribution of points drawn iid from $\D(\th)$.
    \item $\ell(z; \th)$ denotes the loss of the model with parameters $\th$ on the point $z$. For regression problems, this will typically be the (regularized) square loss; for (binary) classification problems, this will typically be the (regularized) cross-entropy loss.
    \item $L(\th_1, \th_2)$ denotes the decoupled performative loss: $$L(\th_1, \th_2) = \E_{\D(\th_2)} [\ell(z; \th_1)].$$ Note that $\th_1$ denotes the \emph{model's} parameters, while $\th_2$ denote's the \emph{distribution's} parameters.
    \item $\L(\th) = L(\th, \th)$ denotes the performative loss.
    \item It will be convenient to distinguish the two components of the \emph{performative gradient} $\nb_\th \L(\th).$ We denote $\nb_1 \L(\th) = \nb_{\th_1}L(\th_1, \th_2)|_{\th_i = \th}$ and $\nb_2 \L(\th) = \nb_{\th_2}L(\th_1, \th_2)|_{\th_i = \th}$, so $\nb \L = \nb_1 \L + \nb_2 \L$.
    %\LY{maybe useful to state $\nb_\th \L(\th) = \nb_1 \L(\th) + \nb_2 \L(\th)$}.
    \item We denote $\opt = \argmin_{\th \in \Th} \L(\th)$.
    \item $\th_{\mathrm{ALG}}$ denotes the final output of the algorithm ALG. The three algorithms we will consider in this paper are repeated risk minimization (RRM), repeated gradient descent (RGD), and our algorithm, performative gradient descent (PerfGD).
    %\LY{redundant}
\end{itemize}
Using the above notation, our interaction model is as follows. Start with some initial model parameters $\th_0$ and observe data $(z^t_i)_{i=1}^n \iid \D(\th_0)$. Then for $t=0,1\ldots T-1$, compute $\th_{t+1}$ using only information from the previous model parameters $\th_s, s\leq t$ and datasets $(z^s_i)_{i=1}^n$, $s\leq t$. The goal of performative ML is to efficiently compute model parameters $\hat{\th} \approx \opt$. For our purposes, we will mainly consider the number of model deployments $T$ as our measure of efficiency, and our goal is to keep this number of deployments low. This corresponds to a setting where deploying a new model is costly, but once the model has been deployed the marginal cost of obtaining more data and performing computations is low.

\subsection{Previous algorithms}
The authors of \cite{Perdomo2020} formalized the performative prediction problem and introduced two algorithms---repeated risk minimization (RRM) and repeated gradient descent (RGD)---for computing a near-optimal point. We introduce these algorithms below.
\begin{algorithm}
\caption{Repeated gradient descent (RGD) \cite{Perdomo2020}}\label{alg: rgd}
\begin{algorithmic}
\WHILE{not converged}
    \STATE Draw $z^{(t)}_i \iid \D(\th_t)$, $i = 1,\ldots,n_t$.
    \STATE $\hnb_1\L(\th_t) \gets \frac{1}{n_t}\sum_{i=1}^{n_t} \nb \ell(z^{(t)}_i; \th_t)$
    \STATE $\th_{t+1} \gets \th_t - \eta_t \hnb_1\L(\th_t)$
    \STATE $t \gets t+1$
\ENDWHILE
\end{algorithmic}
\end{algorithm}

\begin{algorithm}
\caption{Repeated risk minimization (RRM) \cite{Perdomo2020}}\label{alg: rrm}
\begin{algorithmic}
\WHILE{not converged}
    \STATE Draw $z^{(t)}_i \iid \D(\th_t)$, $i = 1,\ldots,n_t$.
    \STATE $\th_{t+1} \gets \argmin_{\th} \sum_{i=1}^{n_t} \ell(z^{(t)}_i; \th)$
    \STATE $t \gets t+1$
\ENDWHILE
\end{algorithmic}
\end{algorithm}
The authors show that under certain assumptions on the loss and distribution shift, RRM and RGD converge to a \emph{stable} point (i.e. model parameters $\stab$ such that $\stab = \argmin_{\th} L(\th, \stab)$), and that $\stab \approx \opt$. When these assumptions fail, however, RRM and RGD may converge to a point very far from $\opt$, or may even fail to converge at all.

\subsection{Why aren't previous algorithms sufficient?} \label{sec: rrm bad}

%Example 3.4 from \cite{Perdomo2020} showcases this latter phenomenon. To quote directly from their paper (but using our notation): Consider optimizing the squared loss $\ell(z; \th) = (z - \th)^2$, where $\th \in [0,1]$ and the distribution of the outcome $z$, according to $\D(\th)$, is a point mass at $0$ if $\th \geq 1/2$, and a point mass at $1$ if $\th < 1/2$. Clearly there is no performatively stable point, and RRM will simply result in the alternating sequence $1,0,1,0,\ldots$. The performative optimum in this case is $\opt = 1/2$.

As an example, let $\Z = \R$, $\Th = [-1, 1]$, and $\ell(z; \th) = z\th$. Define the distribution map $\D(\th) = \N(a_1\th + a_0, \s^2)$ for some fixed $\s^2$. The performative loss is then given by 
\begin{equation*}
    \min_{\th \in [-1, 1]} \E_{\N(a_1\th + a_0, \s^2)}[\th z] = \min_{\th \in [-1, 1]} a_1\th^2 + a_0\th.
\end{equation*}
The optimal solution is at $\opt = -a_0/2a_1$. Let us analyze the behavior of RRM. Letting $(z^t_i)_{i=1}^n$ denote the data sampled from $\D(\th_t)$ and $\bar{z}^t_i = \frac1n\sum_{i=1}^n z^t_i$, RRM will set $\th_{t+1} = \I\{\bar{z}^t_i < 0\} - \I\{\bar{z}^t_i \geq 0\}$. If $a_1 > a_0 \geq 0$ and a sufficient number of samples are drawn at each deployment, and assuming $\s^2$ is small, with high probability when $\th_t = 1$, we will have $\bar{z}^t_i \approx a_1 + a_0 > 0 \Rightarrow \th_{t+1} = -1$, and when $\th_t = -1$, we will have $\bar{z}^t_i \approx -a_1 + a_0 < 0 \Rightarrow \th_{t+1} = 1$. That is, RRM will oscillate between $\th = \pm 1$ and fail to converge even to a stable point.

Next we analyze RGD. At each step, we update $\th_{t+1} = \th_t - \n \E_{\D(\th_t)}\nb [\th_t z]$. If this procedure converges, it will converge to a point $\stab$ such that $\E_{\D(\stab)}\nb [\stab z] = 0$. We can evaluate this expectation explicitly, and we see that $\stab = -a_0/a_1 = 2\opt$. Thus we see that in this simple case, RRM and RGD fail to find the optimal point, motivating our search for improved algorithms. In Section \ref{sec: pricing}, we will return to a more general version of the problem introduced above and verify that our method, PerfGD, does indeed converge to $\opt$.

\section{General formulation of PerfGD} \label{sec: parametric}
Our main goal is to devise a more accurate estimate for the true performative gradient $\nb \L = \nb_1 \L + \nb_2 \L$. We already have a good stochastic estimate for $\nb_1 \L$ (this is just the gradient used by RGD), so we just need to estimate $\nb_2\L$, i.e. the part of the gradient which actually accounts for the shift in the distribution.
%\james{Good to explain how the approach works for a general parametrized distribution, and then say for concreteness we focus on mixture of Gaussians in our experiments.}

In order to accomplish this, we make some parametric assumptions on $\D(\th)$. Namely, we will assume that $\D(\th)$ has a continuously differentiable density $p(z; f(\th))$, where the functional form of $p(z; w)$ is known and the quantity $f(\th)$ is easily estimatable from a sample drawn from $\D(\th)$. For instance, if $\D(\th)$ is in an exponential family, it has a density of the form $\frac{h(z)\exp[\n(\th)^\T T(z)]}{\int h(y)\exp[\n(\th)^\T T(y)]\, dy}$, which corresponds to the known function
$$p(z; w) = \frac{h(z)\exp[w^\T T(z)]}{\int h(y)\exp[w^\T T(y)]\, dy}$$
and unknown function $f(\th) = \n(\th)$. For standard exponential families, there is a straightforward method of estimating the natural parameters $\n(\th)$ from a sample from $\D(\th)$. Thus any exponential family fits within this framework.

For concreteness, for the majority of the paper we will assume that $\D(\th) = \sum_{i=1}^K \g_i \N(\mu_i(\th), \Sigma_i)$, $\sum_{i=1}^k \g_i = 1$, $\g_i \geq 0$ is a mixture of normal distributions with varying means and fixed covariances.
%\james{Do we want to use mixture of Gaussians instead of one Gaussian here? The next sentence talks about mixture.}
As any probability distribution with a smooth density can be approximated to arbitrary precision via a mixture of Gaussians, we will see that this parametric assumption on $\D(\th)$ gives rise to a very powerful method.

\subsection{Algorithm description} \label{sec: general alg description}
To describe the algorithm, it will be convenient to introduce some notation. For any collection of vectors $v_0, v_1,\ldots \in \R^p$ and any two indices $i<j$, we will denote by $v_{i:j}$ the matrix whose columns consist of $v_i, v_{i+1}, \ldots, v_j$, i.e.
\begin{equation}
    v_{i:j} = \begin{bmatrix} | & | & & | \\ v_i & v_{i+1} & \cdots & v_j \\ | & | & & | \end{bmatrix} \in \R^{p \times (j - i + 1)}.
\end{equation}
We also define $\textbf{1}_H \in \R^H$ to be the vector consisting of $H$ ones. Recalling that the space of model parameters $\Th$ is assumed to be closed and convex, we define $\mathrm{proj}_\Th(\th)$ to be the Euclidean projection of $\th$ onto $\Th$. Using this notation the pseudocode for PerfGD is given by Algorithm \ref{alg: general PerfGD}.
\begin{algorithm}[!ht]
   \caption{PerfGD}
   \label{alg: general PerfGD}
\begin{algorithmic}
   \STATE {\bfseries Input:} Learning rate $\n$; gradient estimation horizon $H$; parametric estimator function $\hf$; gradient estimator function $\hnb \L_2$ \\[10pt]
   
   \STATE Take first $H$ updates via RGD
   \FOR{$t = 0$ {\bfseries to} $H - 1$}
   \STATE Draw a new sample and compute estimate for $f(\th_t)$
   \STATE $(z_i)_{i=1}^n \iid \D(\th_t)$
   \STATE $f_t \leftarrow \hf((z_i)_{i=1}^n)$ \\[10pt]
   
   \STATE Compute naive gradient estimate and update parameters
   \STATE $\nb_1\L \leftarrow \frac1n\sum_{i=1}^n \nb \ell(z_i; \th_t)$
   \STATE $\th_{t+1} \leftarrow \mathrm{proj}_{\Th}(\th_t - \n \nb_1\L)$
   \ENDFOR \\[10pt]
   %\james{what is $\Th$ and is it necessary to introduce this and $\mathrm{proj}$ here?}
   %{\color{red} $\Th$ is the parameter space for the models, defined in Section \ref{sec: notation}. Maybe it would be better to not specify that we're using a projected gradient descent update and just allow and gradient-based parameter update here.}
   
   \STATE Run gradient descent with full gradient estimate
   \WHILE{not converged}
   \STATE Draw a new sample and compute estimate for $f(\th_t)$
   \STATE $(z_i)_{i=1}^n \iid \D(\th_t)$
   \STATE $f_t \leftarrow \hf((z_i)_{i=1}^n)$ \\[10pt]
   
   \STATE Estimate the first part of the performative gradient
   \STATE $\nb_1\L \leftarrow \frac1n\sum_{i=1}^n \nb \ell(z_i; \th_t)$ \\[10pt]
   
   \STATE Estimate the second part of the performative gradient 
   \STATE $\Delta \th \leftarrow \th_{t - H: t - 1} - \th_t \textbf{1}_H^\T$
   \STATE $\Delta f \leftarrow f_{t - H : t - 1} - f_t \textbf{1}_H^\T$
   \STATE $\frac{\Delta f}{\Delta\th} \leftarrow (\Delta f) (\Delta \th)^\dagger$
   \STATE $\nb_2\L \leftarrow \hnb\L_2(f_t, \frac{\Delta f}{\Delta\th})$ \\[10pt]
   
   \STATE Update the model parameters
   \STATE $\th_{t+1} \leftarrow \mathrm{proj}_{\Th}(\th_t - \n (\nb_1\L + \nb_2\L))$
   \STATE $t \leftarrow t + 1$
   
   \ENDWHILE
\end{algorithmic}
\end{algorithm}

\subsection{Derivation}
Assume that $\D(\th)$ has density $p(z; f(\th))$ with $p(z; w)$ known for arbitrary $w$. The performative loss is given by $\L(\th) = \int \ell(z; \th) p(z; f(\th)) \, dz$. Assuming that $p$ and $f$ are continuously differentiable, we can compute the performative gradient:
\begin{equation} \label{eq: general PerfGD grad}
    \nb\L(\th) = \underbrace{\int \nb \ell(z; \th) p(z; f(\th)) \, dz}_{\nb_1\L}  + \underbrace{\int \ell(z; \th) \frac{df}{d\th}^\T \partial_2 p(z; f(\th)) \, dz}_{\nb_2\L}.
\end{equation}
Note that $\nb_1\L = \E_{\D(\th)}[\nb \ell(z; \th)]$ and we can obtain an estimate for this quantity by simply averaging $\nb\ell$ over our sample from $\D(\th)$. For $\nb_2\L$, the only unknown quantities are $f(\th)$ and $df/d\th$. By assumption, $f(\th)$ should be easily estimatable from our sample, i.e. there exists an estimator function $\hf$ which, given a sample $(z_i)_{i=1}^n \iid \D(\th)$ returns $\hf((z_i)_{i=1}^n)\approx f(\th)$.

To estimate $df/d\th$, we use a finite difference approximation. By Taylor's theorem, we have $\Delta f \approx \frac{df}{d\th} \Delta \th$.
By taking a pseudoinverse of $\Delta \th$, we obtain an estimate for the derivative: $\frac{df}{d\th} \approx \Delta f (\Delta \th)^\dagger$.
%For all of our experiments, we used the standard \texttt{numpy.linalg.pinv} function \cite{NumPy}, which computes the pseudoinverse using a singular value composition and dropping all singular values below a threshold ($10^{-15}$ by default). 
We require that this this system is overdetermined, i.e. $H \geq p$, to avoid overfitting to noise in the estimates of $f$ and bias from the finite difference approximation to the derivative. (Recall that $H$ is the number of previous finite differences used to estimate $df/d\th$, and $p$ is the dimension of $\th$.)
%\LY{you always make sure that eqn (3) is an overdetermined system, right? otherwise, one will be fitting noise.}
%\LY{shall we say a bit about the pseudo-inverse in the algorithm. For example, how to choose the regularization?}

Substituting these approximations for $f(\th)$ and $df/d\th$ into the expression for $\nb_2\L$, we can then evaluate or approximate the integral using our method of choice. One universally applicable option is to use a REINFORCE-style approximation \cite{REINFORCE}:
\begin{align}
    \nb_2 \L &= \int \ell(z; \th) \frac{df}{d\th}^\T \partial_2[\log p(z; f(\th))] p(z; f(\th)) \, dz \nonumber \\
    &= \E_{\D(\th)}\left[ \ell(z; \th) \frac{df}{d\th}^\T \partial_2[\log p(z; f(\th))] \right]. \label{eq: REINFORCE approximation}
\end{align}
Since $p$ is known, $\partial_2 \log p$ is known as well, and we can approximate equation (\ref{eq: REINFORCE approximation}) by averaging the expression in the expectation over our sample $(z_i)_{i=1}^n$, substituting our approximations for $f(\th)$ and $df/d\th$. Any technique which gives an accurate estimate for $\nb_2\L$ is also acceptable, and we will see in the case of a Gaussian distribution that a REINFORCE estimator of the gradient is unnecessary. We refer to the approximation of the full gradient $\nb\L = \nb_1 \L + \nb_2 \L$ obtained by this procedure as $\hnb \L$.

\section{Theoretical results} \label{sec: theory}
In this section, we quantify the performance of PerfGD theoretically. For simplicity, we focus on the specific case where $\D(\th) = \N(f(\th), \s^2)$ is a one-dimensional Gaussian with fixed variance, and our model also has a single parameter $\th \in \R$. We also use a single previous step to estimate $df/d\th$ (i.e. $H = 1$). For results with longer estimation horizon $(H > 1)$ and stochastic errors on $\hf$, see Appendix \ref{sec: long horizon convergence}.

Below we state our assumptions on the mean function $f$, the loss function $\ell$, and the errors on our estimator $\hf$ of $f$.
\begin{enumerate}
    \item We assume that $f$ has bounded first and second derivatives: $|f'(\th)| \leq F$ and $|f''(\th)| \leq M$ for all $\th \in \R$. \label{cond: f derivs}
    \item The estimator $\hf$ for $f$ has bounded error: $\hf(\th) = f(\th) + \e(\th)$ and $|\e(\th)| \leq \d$. \label{cond: f err}
    \item The loss is bounded: $|\ell(z; \th)| \leq \lmax$. \label{cond: loss}
    \item The gradient estimator $\hnb \L$ is bounded from below and above: $g \leq |\hnb \L| \leq G$. \label{cond: perf grad}
    \item The true performative gradient is upper bounded by $G$: $|\nb \L| \leq G$. \label{cond: true perf grad}
    \item The true performative gradient is $\lip$-Lipschitz: $| \nb \L(\th) - \nb \L(\th') | \leq \lip | \th - \th' |$.
    \item The performative loss is convex.
\end{enumerate}
Lastly, we assume that all of the integrals and expectations involved in computing $\hnb \L$ are computed exactly, so the error comes only from the estimate $\hf$ and the finite difference used to approximate $df/d\th$.

We will prove that PerfGD converges to an approximate critical point, i.e. a point where $\nb\L \approx 0$. The lower bound in condition \ref{cond: perf grad} can therefore be thought of as a stopping criterion for PerfGD, i.e. when the gradient norm drops below the threshold $g$, we terminate. As a corollary to our main theorem, we will show that this criterion can be taken to be $g \propto \d^{1/5}$. We begin by bounding the error of our approximation $\hnb \L_t$. In what follows, $\nb \L_t = \nb \L(\th_t)$ and $\hnb \L_t = \hnb \L(\th_t)$.
\begin{lemma} \label{thm: gradient error bound}
With step size $\n$, the error of the performative gradient is bounded by
\begin{equation*}
    |\hat{\nb}\L_t - \nb\L_t| = \O\left(\lmax \left(MG\n + \frac{\d}{g} \frac1\n + F\d \sqrt{\log \frac1\d} \right) \right).
\end{equation*}
\end{lemma}
Next, we quantify the convergence rate of PerfGD as well as the error of the final point to which it converges.
\begin{theorem} \label{thm: convergence}
With step size
$$\n = \sqrt{\frac{1}{MG^2 T} + \frac{\d}{MGg}},$$
the iterates of PerfGD satisfy
$$\min_{1 \leq t \leq T}| \nb \L_t |^2 = \max\left\{\O\left(\lmax \sqrt{\frac{MG^2}{T} + \frac{MG^3\d}{g}}\right), \: \O(g^2 + \e_*) \right\},$$
where $\e_* = (T^{-1} + \d) \cdot \O(\mathrm{poly}(\lmax, M, G, g^{-1}))$.
\end{theorem}

\blockcomment{
\begin{theorem} \label{thm: convergence}
Let $\L_t = \L(\th_t)$. Then there exist constants $c_1, c_2, c_3 > 0$ such that with step size $\n = \O(\d^{c_1}/L)$, we have
\begin{equation*}
    \min_{1 \leq t \leq T}| \nb \L_t |^2 = \max\left\{\O\left(\frac{L\lmax}{T} + \e_*\right),  \O(g^2 + \d^{c_2}) \right\},
\end{equation*}
where $\e_* = \d^{c_3}\cdot \O(\mathrm{poly}(\lmax, g^{-1}, G, L, F, M))$.
\end{theorem}
}
%\james{Explain this result more. What is $f$ and $\epsilon_*$? Is this using the step size $\eta$ in the previous Theorem?}
Theorem \ref{thm: convergence} shows that PerfGD converges to an approximate critical point. A guarantee on the gradient norm can easily be translated into a bound on the distance of $\th_t$ to $\opt$ with mild additional assumptions. For instance, if $\L$ is $\alpha$-strongly convex, then a standard result from convex analysis implies that $|\th_t - \opt| \leq \alpha^{-1}|\nb \L_t|$. The proof amounts to combining the error bound from Lemma \ref{thm: gradient error bound} with a careful analysis of gradient descent for $\lip$-smooth functions. For details, see the appendix.

Lastly, as as a corollary to Theorem \ref{thm: convergence}, we see that we can choose the stopping criterion to be $g \propto \d^{1/5}$.

\begin{corollary}
With stoppping criterion $g \propto \d^{1/5}$, the iterates of PerfGD satisfy
$$\min_{1\leq t \leq T} |\nb\L_t|^2 = \O\left(\lmax \sqrt{\frac{MG^2}{T} + MG^3 \d^{4/5}} \right).$$
In particular, this suggests that the error in PerfGD will stop decaying after approximately $T \propto \d^{-4/5}$ iterations.
\end{corollary}

The corollary follows trivially from the expression for $\e_*$ and by matching the leading order behavior in $\d$ of the two terms in the max in Theorem \ref{thm: convergence}.

%%%%%%%%%%%%%%%%%%%%%%%%%%%%%%%%%%%%%%%%%%%%%%%%%%%%%%%%%%%%%%%%%%%%%%%%%%%%%%%%%%%%%%%%%%%%%%%%%%%%%%%%%%%%%%%%%%%%%%%%%%%%%%%%%%%%%%%%%%%%%%%%%%%%%%%%%%%%%%%%%%%%%%%%
%%%%%%%%%%%%%%%%%%%%%%%%%%%%%%%%%%%%%%%%%%%%%%%%%%%%%%%%%%%%%%%%%%%%%%%%%%%%%%%%%%%%%%%%%%%%%%%%%%%%%%%%%%%%%%%%%%%%%%%%%%%%%%%%%%%%%%%%%%%%%%%%%%%%%%%%%%%%%%%%%%%%%%%%

\section{Applying PerfGD} \label{sec: applications}
In this section we will show by way of several examples that this simple framework can easily handle performative effects in many practical contexts. For concreteness, we will work with Gaussian distributions with fixed covariance, i.e. $\D(\th) = \N(\mu(\th), \Sigma)$. Using the terminology from Section \ref{sec: general alg description}, for a $d$-dimensional Gaussian we have $p(z; w) = \frac{1}{\sqrt{(2\pi)^d \det \Sigma}}e^{-\frac12 (z - w)^\T \Sigma^{-1} (z - w)}$ and $f(\th) = \mu(\th)$ is the mean of the Gaussian. Our estimator $\hat{\mu}$ for $\mu(\th)$ is just the sample average: $\hat{\mu}((z_i)_{i=1}^n) = \frac1n\sum_{i=1}^n z_i$. Of particular note is the form that $\nb_2\L$ takes in this case. An elementary calculation yields
\begin{align}
    \nb_2 \L &= \int \ell(z; \th) \frac{d\mu}{d\th}^\T \Sigma^{-1}(z - \mu(\th)) p(z; \mu(th)) \, dz \nonumber \\
    &= \E_{\D(\th)} \left[\ell(z; \th) \frac{d\mu}{d\th}^\T \Sigma^{-1}(z - \mu(\th))\right]. \label{eq: gaussian pgd}
\end{align}
Equation (\ref{eq: gaussian pgd}) shows that we can approximate $\nb_2\L$ by averaging the expression inside the expectation over our sample from $\D(\th)$ without the need for the REINFORCE trick or other more complicated methods of numerically evaluating the integral. Specifically, we have
$$\hnb_2\L(\mu, \frac{d\mu}{d\th}) = \frac1n\sum_{i=1}^n \ell(z_i; \th) \frac{d\mu}{d\th}^\T \Sigma^{-1}(z_i - \mu).$$
We present each of the following experiments in a fairly general form. For all of the specific constants we used for both data generation and training, see the appendix. In all of the figures below, the shaded region denotes the standard error of the mean over 10 trials for the associated curve.

\subsection{Toy examples: Mixture of Gaussians and nonlinear mean} \label{sec: toy}
Here we verify that PerfGD converges to the performatively optimal point for some simple problems similar (but slightly more difficult than) to the one introduced in Section \ref{sec: rrm bad}. In both cases we take $\ell(z; \th) = \th z$ and $\Th = [-1, 1]$.

For the first example, we set $\D(\th) = \N(\mu(\th), \s^2)$ with $\mu(\th) = \sqrt{a_1\th + a_0}$. Since $\mu$ is nonlinear, estimating $d\mu/d\th$ is more challenging. Despite this fact, PerfGD still finds the optimal point. The results are shown in Figure \ref{fig: sqrt} below.
\begin{figure}[h!]
\includegraphics[width=\linewidth]{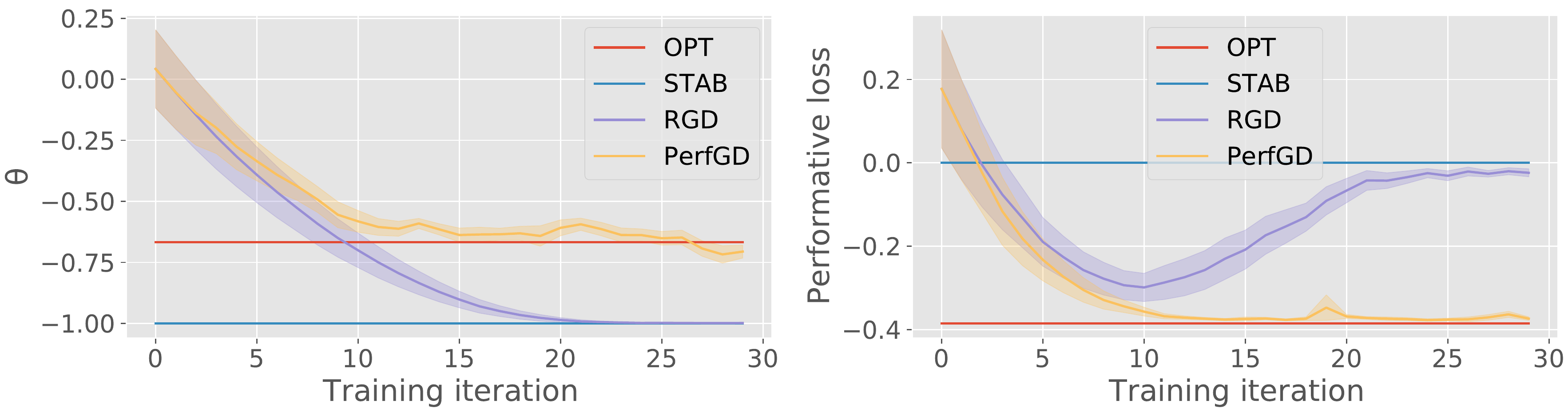}
\caption{PerfGD vs. RGD for a modified version of the toy example introduced in Section \ref{sec: rrm bad}. OPT denotes the performatively optimal point $\opt$, and STAB denotes the performatively stable point $\stab$. We set $\mu(\th) = \sqrt{a_1\th + a_0}$. Since the mean is nonlinear in $\th$, estimating $d\mu/d\th$ with finite differences is more challenging. In spite of this, PerfGD still converges to the optimal point.} \label{fig: sqrt}
%\james{The nonlinear sentence is not clear; is this explained in text? The shading corresponds to ...}
%\LY{maybe change theta to $\theta$ in the axis.} 
%\james{Explain in the caption: OPT is ... STAB is ...}} 
\end{figure}

For the second example, we set $\D(\th) = \g\N(\mu_1(\th), \s_1^2) + (1-\g)\N(\mu_2(\th), \s_2^2)$. Here both of the means are linear in $\th$, i.e. $\mu_i(\th) = a_{i,1} \th + a_{i,0}$.
%We estimate the two means by fitting a Gaussian mixture model with the EM algorithm. The increased error in mean estimation causes some slight error in PerfGD, but it still gets closer to $\opt$ than RRM.
We apply PerfGD where the true cluster assignment for each point are known; in this case, PerfGD converges to $\opt$ exactly and achieves optimal performative loss. The results are shown in Figure \ref{fig: mixture} below.
\begin{figure}[h!]
\includegraphics[width=\linewidth]{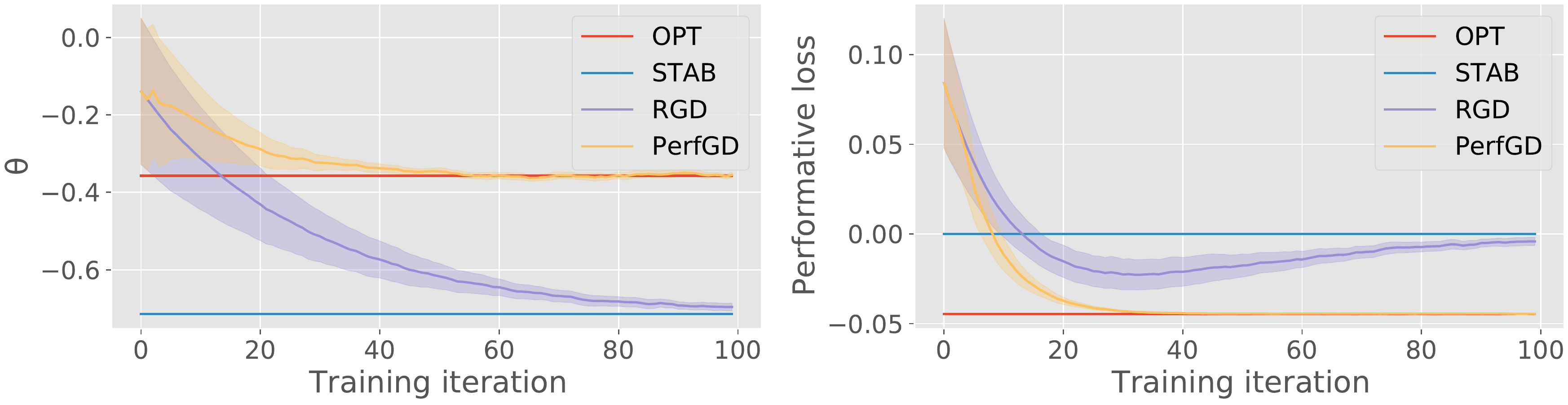}
\caption{PerfGD vs. RGD for a modified version of the toy example introduced in Section \ref{sec: rrm bad} with a mixture of Gaussians. We supply the cluster label for each point. PerfGD is again able to converge to the minimum, while RGD converges to a suboptimal point. 
} \label{fig: mixture}
%\LY{no PerfGD+ in figure}
%\james{Maybe clearer to preset PerfGD+ curve as PerfGD here since it is the same algorithm with a different estimation. Then we can remove the original PerfGD curves and simplify the figure.}} 
\end{figure}

\subsection{Pricing} \label{sec: pricing}
We next examine a generalized version of the problem introduced in Section \ref{sec: rrm bad}. Let $\th$ denote a vector of prices for various goods which we, the distributor, set. A vector $z$ denotes a customer's demand for each good. Our goal is to maximize our expected revenue $\E_{\D(\th)}[\th^\T z]$. (In other words, we set the loss function $\ell(z; \th) = -\th^\T z$.) Assuming $\D(\th) = \N(\mu(\th), \Sigma)$, we can directly apply Algorithm \ref{alg: general PerfGD} with the functions $\hat{\mu}$ and $\hnb \L_2$ defined at the beginning of the section to compute the optimal prices.

\paragraph{Experiments}
For this experiment, we work in a higher dimensional setting with $d = 5$. We define $\Theta = [0, 5]^d$ and $\mu(\th) = \mu_0 - \e \th$. (That is, the mean demand for each good decreases linearly as the price increases.) 
%We also truncate the demand vectors to be nonnegative.
%\LY{need to specify the dimension $d$ of the test?}

Our results are shown in Figure \ref{fig: pricing}. For this case, we can compute $\opt$ and $\stab$ analytically. The performative revenue for each of these points is shown on the right side of the figure. As expected, PerfGD converges smoothly to the optimal prices, while RGD converges to the only fixed point which produces suboptimal revenue. In this case, RRM (not shown) stays fixed at $\rrm = [5, 5, \ldots, 5]^\T$.

\begin{figure}[h!]
\includegraphics[width=\linewidth]{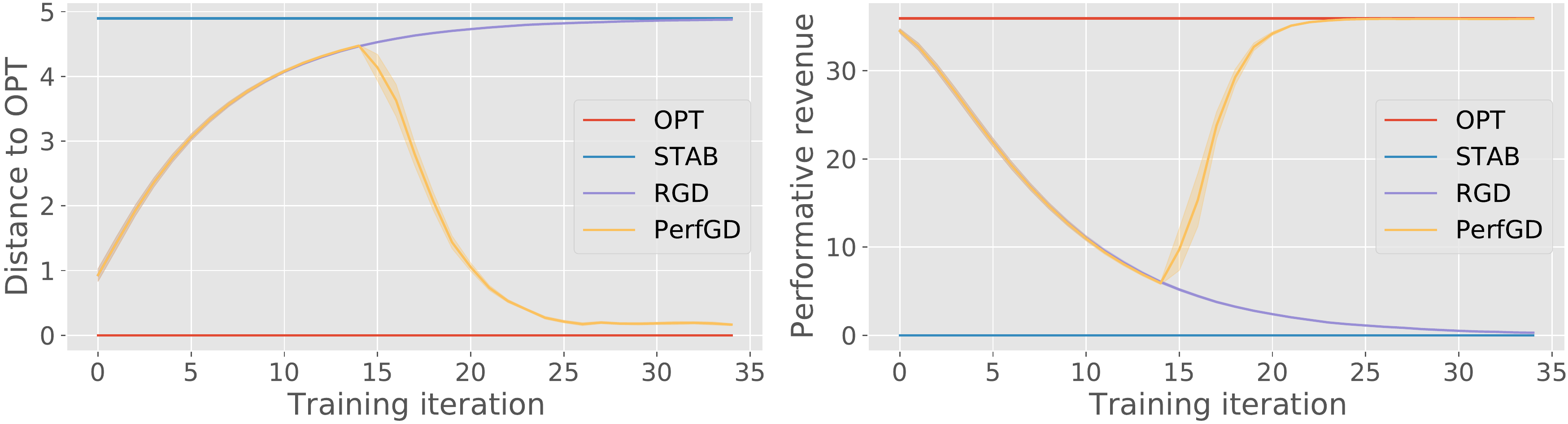}
\caption{PerfGD vs. RGD for performative pricing. By taking into account the change in distribution, PerfGD converges to a set of prices which yields higher revenue than RGD. RRM (not shown) stays fixed at $\rrm = 5 \cdot \textbf{1}$, i.e. the vector with all entries equal to 5. Note that PerfGD follows RGD for the first several steps as part of the initialization phase. After this phase, the accurate estimate for the second part of the performative gradient allows PerfGD to reverse trajectory towards $\opt$.} \label{fig: pricing}
%\james{Add some intuitive explanation for why PerfGD follows RGD for the first half before reversing.}} 
\end{figure}

\subsection{Binary classification} \label{sec: classification}
Suppose our goal is to predict a label $y\in\{0, 1\}$ using features $x\in \R^d$. We assume that the label $y \sim \Bern(\g)$, and that $x | y \sim \N(\mu^y(\th), \Sigma_y)$. The performative loss can then be written as
\begin{equation} \label{eq: classification}
    \L(\th) = (1-\g)\E_{\N(\mu^0(\th), \Sigma_0)}[\ell(x, 0; \th)] + \g \E_{\N(\mu^1(\th), \Sigma_1)}[\ell(x, 1; \th)]
\end{equation}
We can apply the general PerfGD method to each of the terms in (\ref{eq: classification}) to obtain an approximate stochastic gradient. (We treat the features of the data with label $y=0$ as the dataset for the first term, and the features of the data with label $y=1$ as the dataset for the second term.)

\paragraph{Experiments}
Here we work with a synthetic model of the spam classification example. We will classify emails with a logistic model, and we will allow a bias term. (That is, our model parameters $\th = (\th_0, \th_1)^\T \in \R^2$. Given a real-valued feature $x$, our model outputs $h_\th(x) = 1/(1 + e^{-\th_0 - \th_1 x})$.) We let the label $y = \I\{\textrm{email is spam}\}$. For this case we assume that the distribution of the feature given the label is the performative aspect of the distribution map: spammers will try to alter their emails to slip past the spam filter, while people who use email normally will not alter their behavior according to the spam filter. To this end, we suppose that
$$x | y = 0, \th \sim \N(\mu_0, \s_0^2), \hspace{.25in} x | y = 1, \th \sim \N(f(\th), \s_1^2).$$
We note that assuming Gaussian features is in fact a realistic assumption in this case. Indeed, \cite{Li2020} shows that state-of-the-art performance on various NLP tasks can be achieved by transforming standard BERT embeddings so that they look like a sample from an isotropic Gaussian.

For this experiment, we set $f(\th) = \mu_1 - \e \th_1$. Such a distribution map arises from the strategic classification setting described in \cite{Perdomo2020} in which the spammers optimize a non-spam classification utility minus a quadratic cost for changing their features. We use ridge-regularized cross-entropy loss for $\ell$.
%$$\ell(x, y; \th) = -y\log h_\th(x) - (1-y)\log (1 - h_\th(x)) + \frac\lambda2 \lVert \th \rVert^2.$$

Our results are shown in Figure \ref{fig: classification}. The improved estimate of the performative gradient given by PerfGD results in roughly a $9\%$ reduction in the performative loss over RGD. In this case, RRM (not shown) oscillates between two values of $\th$ which both give significantly higher performative loss than either RGD or PerfGD.

\begin{figure}[h!]
\begin{center}
\includegraphics[width=\linewidth]{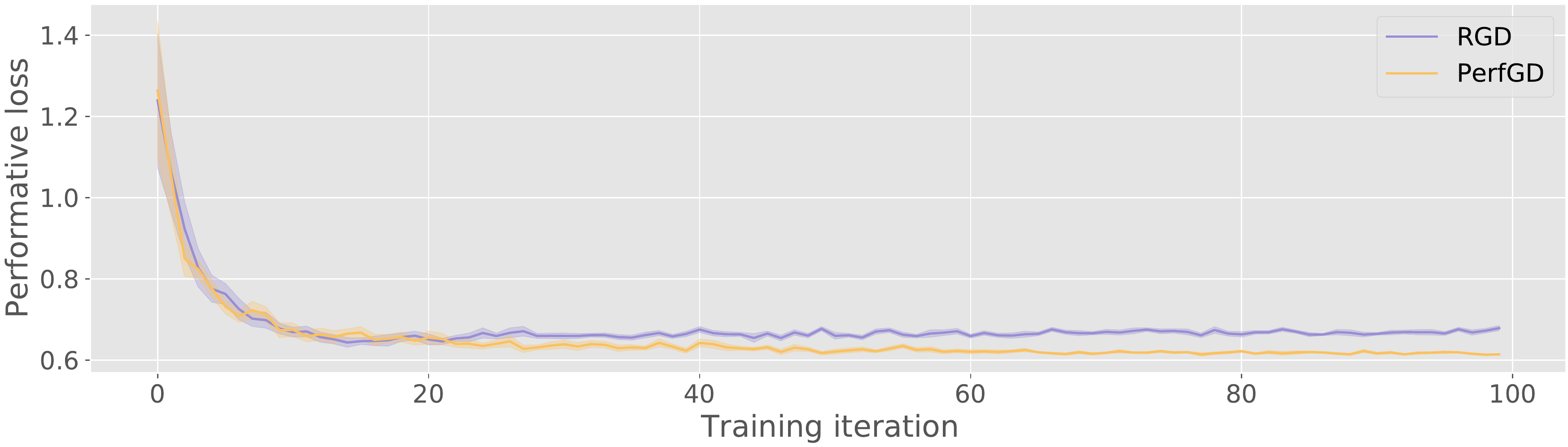}
\caption{PerfGD vs. RGD for performative logistic regression. By taking into account the change in distribution, PerfGD is able to achieve a lower performative loss than RGD.} \label{fig: classification}
\end{center}
\end{figure}

%\paragraph{Real data}
%We use the same real dataset as \cite{Perdomo2020}, which can be found {\color{blue}\href{https://www.kaggle.com/c/GiveMeSomeCredit/data}{here}}. We use the same subsampling and normalizing procedure as \cite{Perdomo2020}: we subsample so that $\sim45\%$ of the data are from the positive class (default), and we normalize all features to have mean 0 and variance 1. The data consist of 10 features, of which we allow three (usage of credit lines, number of open credit lines, and number of real estate loans) to shift due to performative effects. We can account for the fact that only three of the features are performative by zeroing out all of the rows in $d\mu/d\th$ which correspond to the non-performative features, as we know that the true change in mean for these features with respect to $\th$ is zero.

%%%%%%%%%%%%%%%%%%%%%%%%%%%%%%%%%%%%%%%%%%%%%%%%%%%%%%%%%%%%%%%%%%%%%%%%%%%%%%%%%%%%%%%%%%%%%%%%%%%%%%%%%%%%%%%%%%%%%%%%%%%%%%%%%%%%%%%%%%%%%%%%%%%%%%%%%%%%%%%%%%%%%%%%

\subsection{Regression} \label{sec: regression}
This setting is essentially a generalized version of the performative mean estimation problem in \cite{Perdomo2020}. For simplicity, assume that the marginal distribution of $x$ is independent of $\th$. Assuming that $y|x \sim \N(\mu(x, \th), \s^2)$, the performative loss becomes
\begin{equation} \label{eq: regression}
    \L(\th) = \E_x[\E_{\N(\mu(x, \th), \s^2)}[\ell(x, y; \th)]].
\end{equation}
The inner expectation has the required form to apply PerfGD. However, since $x$ takes continuous values, we will in general have only one sample to approximate the inner expectation in (\ref{eq: regression}), leading to heavily biased or innacurate estimates for the required quantities in (\ref{eq: general PerfGD grad}). This leaves us with two options: we can either use techniques for debiasing the required quantities and apply PerfGD directly, or we can use a reparameterization trick and a modified version of PerfGD. Here we present the latter approach.

We assume that the response $y$ follows a linear model, i.e. $y = \beta(\th)^\T x + \e$, $\e \sim \N(0, \s^2)$. The performative loss can then be written as
\begin{equation}
    \L(\th) = \E_{x,\e}[\ell(x, \b(\th)^\T x + \e; \th)].
\end{equation}
Since we have removed the dependence of the distribution on $\th$, we can easily compute the gradient:
\begin{equation*}
    \nb \L(\th) = \E_{\D(\th)}[\nb_\th \ell(x, y; \th)] + \E_{\D(\th)}\left[\frac{\partial \ell}{\partial y} \frac{d\b}{d\th}^\T x \right].
\end{equation*}
We can first estimate $\b$ via e.g. regularized ordinary least squares, then estimate $d\b/d\th$ via finite differences as in the general setting (\ref{eq: general PerfGD grad}): $\frac{d\b}{d\th} \approx \Delta \b (\Delta \th)^{\dagger}$.

\paragraph{Experiments}
For simplicity, we use one-dimensional linear regression parameters $\th \in \R$. The feature $x$ is drawn from a fixed distribution $x\sim \N(\mu_x, \s_x^2)$, and the performative coefficient $\b(\th)$ of $y|x$ has the form $\b(\th) = a_0 + a_1\th$. We use ridge-regularized squared loss for $\ell$.
%$$\ell(x, y; \th) = \frac12(\th x - y)^2 + \frac\lambda2|\th|^2.$$

%Here we took samples of $n=500$ data points with a learning rate of $\n = 0.1$. We used the full history to approximate $\b'$.

Our results are summarized in Figure \ref{fig: regression}. In this case, there is a large gap between $\opt$ and $\stab$. As expected, PerfGD converges smoothly to $\opt$, while in this case both RGD and RRM converge to $\stab$. The improvement of PerfGD over RGD and RRM results in a factor of more than an order of magnitude in reduction of the performative loss.
\begin{figure}[h!]
\includegraphics[width=\linewidth]{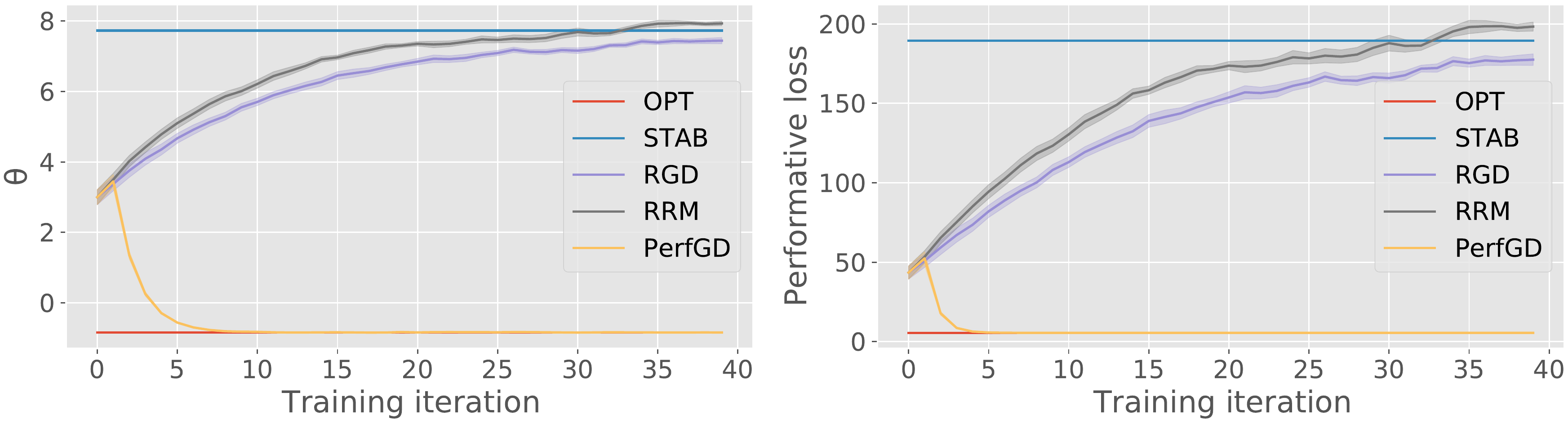}
\caption{PerfGD vs. RGD for performative linear regression. \textbf{Top:} Model parameters
  vs. training iteration. \textbf{Bottom:} Performative loss vs. training iteration.  As expected,
  RGD converges to the performatively stable point, but in this case the stable point is very far
  from the performative optimum. PerfGD converges to OPT and incurs a much lower performative loss
  than RGD.} \label{fig: regression}
  %\LY{shall be top and bottom instead of left and right.}
\end{figure}

\section{Conclusion} \label{sec: conclusion}
In this paper, we addressed the setting of modeling when the data distribution reacts to the model's
parameters, i.e. performative distribution shift. We verified that existing algorithms meant to
address this setting in general converge to a suboptimal point in terms of the performative loss. We
then introduced a new algorithm, PerfGD, which computes a more accurate estimate for the
performative gradient under some parametric assumptions on the performative distribution. We proved
theoretical results on the accuracy of our gradient estimate as well as the convergence of the
method, and confirmed via several empirical examples that PerfGD outperforms existing algorithms such
as repeated gradient descent and repeated risk minimization. The accuracy and iteration requirement
are both practically feasible, as many ML systems have regular updates every few days.

Finally, we suggest directions for further research. A natural direction for future work is the extension of our methods to nonparametric distributions. Another direction which may prove fruitful is to improve the estimation of the derivative $df/d\th$. Finally, methods specifically tailored to deal with high-dimensional data are also of interest.

%\james{Pre-empt concerns about PerfGD requiring more iterations than RRM: many ML systems have regular updates every few days so 50 or even 100 iterations are reasonable.  }

\newpage
\nocite{*}
\bibliography{pgd}
\bibliographystyle{alpha}

\newpage
\appendix
Throughout the following proofs, we use $\O(\cdot)$ to denote the leading order behavior of various quantities as $T$ (the total number of steps taken by the method) becomes large and $\d$ (the size of the error in the estimate for $f$) becomes small. For simplicity, all proofs are in one dimension.

\section{\texorpdfstring{Bounding the error of $\hfp$}{}}
Before we prove Lemma \ref{thm: gradient error bound} (bounding the error of the full performative gradient), we must first bound the error of our approximation to $f'$. Let $f_t = f(\th_t)$, $f'_t = f'(\th_t)$, and define $\hft$ and $\hfpt$ similarly.

\begin{lemma} \label{thm: f' error bound}
Under the assumptions of Section \ref{sec: theory}, we have $|\hfpt - f'_t| = \O\left(\frac{\d}{g}\frac1\n + MG\n\right)$.
\end{lemma}

\begin{proof}
By definition, we have
\begin{equation} \label{eq: taylor for fhat'}
    \hfpt = \frac{\hf_{t+1} - \hft}{\th_{t+1} - \th_t} = \frac{f_{t+1} - f_t}{\th_{t+1} - \th_t} + \frac{\e_{t+1} - \e_t}{\th_{t+1} - \th_t}. 
\end{equation}
By Taylor's theorem, we have
\begin{equation} \label{eq: taylor for f'}
    f_{t+1} = f_t + f'_t\cdot(\th_{t+1} - \th_t) + \frac12 f''(\xi) (\th_{t+1} - \th_t)^2 \hspace{.1in} \Longrightarrow \hspace{.1in} f'_t = \frac{f_{t+1} - f_t}{\th_{t+1} - \th_t} + \frac12 f''(\xi)(\th_{t+1} - \th_t),
\end{equation}
where $\xi$ is some number between $\th_t$ and $\th_{t+1}$. Next, note that since $\th_{t+1} = \th_t - \n \hnb\L_t$ and $g \leq |\hnb \L_t| \leq G$, we have $\n g \leq |\th_{t+1} - \th_t| \leq \n G$. Using this fact and combining equations (\ref{eq: taylor for fhat'}) and (\ref{eq: taylor for f'}), we find that
\begin{align*}
    |\hfpt - f'_t| &\leq \frac{|\e_{t+1}| + |\e_t|}{\n g} + \frac12 |f''(\xi)|\n G \\[10pt]
    &\leq \frac{2\d}{g}\frac1\n + \frac{MG}{2} \n 
\end{align*}
where we have also used the assumption that $|f''(\xi)| \leq M$. This is the desired bound.
\end{proof}

\section{Proof of Lemma \ref{thm: gradient error bound}}

\begin{proof}
We write $|\hnb \L_t - \nb \L_t| \leq |\hnb_1 \L_t - \nb_1 \L_t| + |\hnb_2 \L - \nb_2 \L_t|$ and bound each term on the right-hand side separately. We begin by bounding the error on $\nb_1\L$. We have
\begin{align}
    |\hnb_1 \L_t - \nb_1 \L_t| & \leq \int |\nb \ell(z; \th)| |p(z; \hft) - p(z; f_t)| \, dz \nonumber \\
    &\leq \lmax \int |p(z; \hft) - p(z; f_t)| \, dz \nonumber \\
    &\leq \lmax \bigg(\underbrace{\int_{|z - f_t| \leq R} |p(z; \hft) - p(z; f_t)| \, dz}_{\mathrm{(A)}} + \underbrace{\int_{|z - f_t| > R} |p(z; \hft)| \, dz}_{\mathrm{(B)}} + \underbrace{\int_{|z - f_t| > R} |p(z; f_t)| \, dz}_{\mathrm{(C)}}\bigg), \label{eq: perf grad 1 bound}
\end{align}
where for simplicity we assume that $\lmax \geq |\nb \ell(z; \th)|$ is also an upper bound on the derivative of the point loss, and for any $R > 0$.

To bound (A), we bound the Lipschitz constant of $p$ in its second argument. It suffices to bound $\partial_2 p$. Observe that
$$ \partial_2 p(z; w) = c (z - w)e^{-\frac{1}{2\s^2} (z-w)^2}. $$
Letting $x = z-w$ and $\alpha = \frac{1}{2\s^2}$, we want to bound the maximum of $x e^{-\alpha x^2}$. Taking the derivative with respect to $x$, this has critical points at $x = \pm \frac{1}{\sqrt{2\alpha}}$. Since $|\partial_2 p(z; w)| \rightarrow 0$ as $w\rightarrow \pm\infty$ for any $z$, these critical points are global maxima for $|\partial_2 p|$. Thus $\max |\partial_2 p| = \O(1)$ and $p$ is $\O(1)$-Lipschitz in its second argument. It follows that
\begin{equation*}
    \mathrm{(A)} \leq \int c |\hft - f_t| \, dz = \O(R\d).
\end{equation*}

To bound (B), oberserve that
\begin{align*}
    \mathrm{(B)} &\leq \int_{|z - \hft| + |\hft - f_t| > R} |p(z; \hft)|\, dz \\
    &\leq \int_{|z - \hft| > R - \d} |p(z; \hft)|\, dz \\
    &= \P_{\N(\hft, \s^2)}(|z - \hft| > R - \d) \\
    &\leq 2e^{-(R - \d)^2/2\s^2}
\end{align*}
for any $R \geq \d$. A similar calculation shows that (C) $\leq 2e^{-R^2/2\s^2} \leq 2e^{-(R-\d)^2/2s^2}$ for $R \geq \d$. Thus
$$\mathrm{(A)} + \mathrm{(B)} + \mathrm{(C)} = \O\left(R\d + \exp\left(-\frac{(R-\d)^2}{2\s^2}\right)\right)$$
for any $R \geq \d$. Setting $R = \d + \sqrt{2\s^2 \log\frac1\d}$ and substituting our bound back into (\ref{eq: perf grad 1 bound}), we obtain
\begin{equation} \label{eq: final grad1 bound}
    |\hnb_1 \L_t - \nb_1 \L_t| = \O\left(\lmax\left(\d \sqrt{\log\frac1\d}\right)\right).
\end{equation}

Next we bound the error $|\hnb_2 \L_t - \nb_2\L_t|$. We have
\begin{align}
    |\hnb_2 \L_t - \nb_2 \L_t| &= \left|\int \ell(z; \th_t) \partial_2 p(z; \hft) \hfpt\, dz - \int \ell(z; \th_t) \partial_2 p(z; f_t) f'_t \, dz \right| \nonumber \\[10pt]
    &\leq \underbrace{\int |\ell(z; \th_t)| |\partial_2 p(z; \hft)| |\hfpt - f'_t| \, dz}_{\mathrm{(I)}} + \underbrace{\int |\ell(z; \th_t)| |\partial_2 p(z; \hft) - \partial_2 p(z; f_t)| |f'_t|\, dz}_{\mathrm{(II)}}. \label{eq: bounding performative grad err 1}
\end{align}
We proceed to bound the terms (I) and (II) separately.

The bound for (I) is straightforward. Recall that $|\ell(z; \th_t)| \leq \lmax$ and $\hfpt$ and $f'_t$ are independent of $z$, so we have
\begin{equation*}
    \mathrm{(I)} \leq \lmax |\hfpt - f'_t| \int |\partial_2 p(z; \hft)| \, dz.
\end{equation*}
Since $p(z; \hft)$ is the pdf for a Gaussian with mean $\hft$ and variance $\s^2$, a standard computation reveals that $\int |\partial_2 p(z; \hft)|\, dz = \sqrt{\frac{2}{\pi\s^2}} = \O(1)$. Using the bound on $|\hfpt - f'_t|$ from Lemma \ref{thm: f' error bound}, we have
\begin{equation} \label{eq: bound on (I)}
    \mathrm{(I)} = \O\left(\lmax\left(MG \n + \frac{\d}{g}\frac1\n \right)\right).
\end{equation}

Next, we bound (II). First, since $|\ell(z; \th_t)|\leq \lmax$ and $|f'_t| \leq F$, we have
\begin{equation} \label{eq: bounding (II), 1}
    \mathrm{(II)} \leq \lmax F \int |\partial_2 p(z; \hft) - \partial_2 p(z; f_t)|\, dz
\end{equation}
so it suffices to bound the integrand in (\ref{eq: bounding (II), 1}).

For any $R \geq \d$, we have
\begin{equation*}
    \int |\partial_2 p(z; \hft) - \partial_2 p(z; f_t)|\, dz = \underbrace{\int_{|z - f_t| \leq R} |\partial_2 p(z; \hft) - \partial_2 p(z; f_t)|\, dz}_{\mathrm{(i)}} + \underbrace{\int_{|z - f_t| > R} |\partial_2 p(z; \hft) - \partial_2 p(z; f_t)|\, dz}_{\mathrm{(ii)}}.
\end{equation*}
To bound (i), it suffices to bound the Lipschitz constant of $\partial_2 p(z; w)$ in the second variable (if one exists). We can do this by bounding $|\partial_2^2 p|$. A direct computation shows that
\begin{equation} \label{eq: lipschitz constant of d2p}
    \partial_2^2 p(z; w) = \frac{1}{\s^2 \sqrt{2\pi\s^2}}e^{-\frac{1}{2\s^2}(z-w)^2}\left(\frac{1}{\s^2}(z-w)^2 - 1\right).
\end{equation}
Let $\alpha = \s^{-2}$ and $x = (z-w)^2$. Bounding (\ref{eq: lipschitz constant of d2p}) is equivalent to upper bounding an expression of the form $e^{-\frac{\alpha}{2} x}(\alpha x - 1)$ over $x \geq 0$. Taking a derivative with respect to $x$ shows that the only critical point is at $x = 2e^{-3/2}$; the only other point to check is the boundary point $x=0$. Checking both of these manually shows that the absolute value is maximized at $x=0$, and we obtain the bound
\begin{equation*}
    |\partial_2^2 p(z; w)| \leq \frac{1}{\s^2 \sqrt{2\pi\s^2}} = \O(1),
\end{equation*}
i.e. $\partial_2 p(z; w)$ is $\O(1)$-Lipschitz in $w$. Applying this fact to (i), we have
\begin{equation} \label{eq: bound on (i)}
    \mathrm{(i)} \leq \int_{|z - f_t| \leq R} c|\hft - f_t| \, dz 
    = c \int_{|z - f_t| \leq R} |\e_t|\, dz 
    = \O(R\d).
\end{equation}
Next we turn our attention to (ii). We have
\begin{align}
    \mathrm{(ii)} &\leq \int_{|z - f_t| > R} |\partial_2 p(z; \hft)| \, dz + \int_{|z - f_t| > R} |\partial_2 p(z; f_t)|\, dz \nonumber \\[10pt]
    &\leq \int_{|z-(f_t + \e_t)| > R - |\e_t|} |\partial_2 p(z; \hft)| \, dz + \int_{|z - f_t| > R} |\partial_2 p(z; f_t)|\, dz \nonumber \\[10pt]
    &\leq \int_{|z-f'_t| > R - \d} |\partial_2 p(z; \hft)| \, dz + \int_{|z - f_t| > R} |\partial_2 p(z; f_t)|\, dz. \label{eq: bound on (ii) 1}
\end{align}
These inequalities follow from several applications of the triangle inequality and the bound $|\e_t| \leq \d$. Now since $p(z; w)$ is a Gaussian pdf, we have $\partial_2 p(z; w) = \frac{1}{\s^2} (z-w) p(z; w)$, and therefore
\begin{align}
    \int_{|z-w| > r} |\partial_2 p(z; w)| \, dz &= \int_{|z-w| > r} \frac{1}{\s^2}|z-w| p(z; w) \, dz \nonumber \\[10pt]
    &= \E_{\N(w, \s^2)}\left[ \I\{|z-w| \geq r|\} \s^{-2}|z - w|\right] \nonumber \\[10pt]
    &\leq \sqrt{ \E[\I\{|z-w| \geq r\}^2] \s^{-2}\E[|z-w|^2]} \label{eq: cauchy-schwarz} \\[10pt]
    &= \sqrt{\P(|z-w| \geq r)} \nonumber \\[10pt]
    &\leq \sqrt{2}e^{-\frac{r^2}{4\s^2}}, \label{eq: gaussian tail bound}
\end{align}
where (\ref{eq: cauchy-schwarz}) follows from the Cauchy-Schwarz inequality and (\ref{eq: gaussian tail bound}) follows from a standard Gaussian tail bound. Applying (\ref{eq: gaussian tail bound}) to (\ref{eq: bound on (ii) 1}), we obtain
\begin{align}
    \mathrm{(ii)} &\leq c\left(\exp\left(-\frac{(R - \d)^2}{4\s^2}\right) + \exp\left(-\frac{R^2}{4\s^2}\right)\right) \nonumber \\[10pt]
    &= \O\left(\exp\left(-\frac{(R-\d)^2}{4\s^2}\right) \right) \label{eq: bound on (ii) 2}
\end{align}
for any $R \geq \d$. Combining the bound (\ref{eq: bound on (i)}) on (i) and (\ref{eq: bound on (ii) 2}) on (ii) with (\ref{eq: bounding (II), 1}), we have
\begin{equation} \label{eq: bounding (II), 2}
    \mathrm{(II)} = \O\left(\lmax F\left[R\d + \exp\left\{-\frac{(R - \d)^2}{4\s^2}\right\}\right]\right).
\end{equation}
If we take $R = \d + \sqrt{4\s^2 \log(1/\d)}$ and substitute into (\ref{eq: bounding (II), 2}), we obtain
\begin{equation} \label{eq: bound on (II)}
    \mathrm{(II)} = \O\left( \lmax F \d \sqrt{\log(1/\d)}\right).
\end{equation}
We now substitute our bounds on (I) and (II) into (\ref{eq: bounding performative grad err 1}), which yields
\begin{equation} \label{eq: final grad2 bound}
    |\nb_2 \L - \hnb_2 \L| \leq \O\left(\lmax\left[MG\n + \frac{\d}{g}\frac1\n + F\d\sqrt{\log(1/\d)}\right]\right).
\end{equation}
To conlude, observe that the bound on the error of $\nb_1\L_t$ in (\ref{eq: final grad1 bound}) can be completely absorbed into (\ref{eq: final grad2 bound}), and we obtain the desired result.
\end{proof}

\section{Proof of Theorem \ref{thm: convergence}}
\begin{proof}
To simplify notation, we will let $L = \lip$; this should not be confused with the decoupled performative loss function $L(\th_1, \th_2)$ defined in Section \ref{sec: notation}. Let $\L_t = \L(\th_t)$ and let $E_t = \hnb \L_t - \nb \L_t$. Since $\L$ is $L$-smooth and convex, we have the standard inequality
\begin{equation} \label{eq: L smooth}
    \L_{t+1} \leq \L_t + \nb \L_t \cdot (\th_{t+1} - \th_t) + \frac{L}{2}|\th_{t+1} - \th_t|^2.
\end{equation}
Since $\th_{t+1} - \th_t = \n \hnb \L_t$, we can rewrite (\ref{eq: L smooth}):
\begin{equation} \label{eq: L smooth 2}
    \L_{t+1} \leq \L_t + \n(|\nb \L_t||E_t| - |\nb \L_t|^2) + \n^2 L (|\nb \L_t|^2 + |E_t|^2)
\end{equation}
Rearranging and using the fact that $|\nb \L_t| \leq G$, we have
\begin{equation} \label{eq: grad norm bound}
    (\n - \n^2 L) |\nb\L_t|^2 \leq \L_t - \L_{t+1} + \n G |E_t| + \n^2 L |E_t|^2.
\end{equation}
If we sum both sides of (\ref{eq: grad norm bound}) from $t=1$ to $T$, we find that
\begin{equation} \label{eq: telescope}
    T \min_{1\leq t \leq T} |\nb \L_t|^2 \leq \sum_{t=1}^T |\nb \L_t|^2 \leq \frac{\L_1 - \L_{T+1} + \n G \sum_{t=1}^T |E_t| + \n^2 L \sum_{t=1}^T |E_t|^2}{\n - L\n^2}.
\end{equation}
Note that with $\n = \sqrt{\frac{1}{MG^2 T} + \frac{\d}{MGg}}$ as specified by the theorem, we have $\n^2 = o(\n)$. Furthermore, by Lemma \ref{thm: gradient error bound}, we have
\begin{equation} \label{eq: bound on Et with specific eta}|E_t| = \O\left( \lmax \sqrt{\frac{M}{T} + \frac{MG\d}{g}} \right) \equiv \mathbf{E}.\end{equation}
(In obtaining the above bound, we have assumed WLOG that $G \geq 1$.) Note that since $\mathbf{E} = o(1)$, we have $\mathbf{E}^2 = o(\mathbf{E})$. Lastly, since $\L_t = \E_{p(z; \th_t)}[\ell(z; \th_t)]$ we have $|\L_t| \leq \lmax$ for all $t$. Applying these facts to \eqref{eq: telescope}, we have
\begin{align}
    \min_{1\leq t \leq T} |\nb \L_t|^2 &= \O\left( \frac{ \lmax + \n G T \mathbf{E} + \n^2 L T \mathbf{E}^2}{T\n} \right) \nonumber \\[10pt]
    &= \O\left(\frac{\lmax}{T\n} + G \lmax \left[MG\n + \frac{\d}{g} \frac1\n\right]\right) \nonumber \\[10pt]
    &= \O\left( \lmax \sqrt{\frac{MG^2}{T} + \frac{MG^3 \d}{g}}\right) \label{eq: final perfgd bound}
\end{align}
where the last equation follows from our choice of $\n$.

\blockcomment{
Let us now examine the bound on $|E_t|$ from Lemma \ref{thm: gradient error bound} more carefully. If we take step size $\n = \d^c$ for some $0<c<1$, then we have
\begin{equation*}
    |E_t| = \O\left(\lmax\left[MG\d^c + \frac{\d^{1-c}}{g} + F\d \log(1/\d)\right]\right) = o(1).
\end{equation*}
In particular, this means that have $|E_t|^2 = o(|E_t|)$. Furthermore, when $\n = \d^c$ we have $\n^2 = o(\n)$. Finally, observe that since $\L_t = \E_{p(z; \th_t)}[\ell(z; \th_t)]$ we have $|\L_t| \leq \lmax$ for all $t$. Thus, (\ref{eq: telescope}) can be simplified to
\begin{align}
    \min_{1\leq t \leq T} |\nb \L_t|^2 &= \O\left(\frac{\lmax}{T\d^c} + G \lmax[ M G \d^c + \frac{\d^{1-c}}{g} + F\d \log(1/\d)]\right) \nonumber \\[10pt]
    &= \O\left( \lmax \left[ \frac{1}{T\d^c} + MG^2\d^c + \frac{G\d^{1-c}}{g}\right]\right). \label{eq: final perfgd bound}
\end{align}
}

Lastly, recall that our bound on $|E_t|$ required that $|\hnb \L_t| \geq g$ for all $1\leq t \leq T$. If at any point we have $|\hnb \L_t| < g$, then we can terminate and return this iterate. But then we have
\begin{align}
    |\nb \L_t|^2 &\leq 2|\hnb \L_t|^2 + 2|E_t|^2 \nonumber  \\
    &\leq \O(g^2 + \mathbf{E}^2). \label{eq: small grad terminate}
\end{align}
Note that $\e_* \equiv \mathbf{E}^2 = \O\left(\lmax^2\left(\frac{M}{T} + \frac{MG\d}{g}\right)\right) = (T^{-1} + \d) \cdot \O(\textrm{poly}(\lmax, M, G, g^{-1})$ as specified in the statement of Theorem \ref{thm: convergence}. We can guarantee that PerfGD reaches at least the max of the two bounds (\ref{eq: final perfgd bound}) and (\ref{eq: small grad terminate}), yielding the desired result.

\blockcomment{
By the bound on $E_t$ from Lemma \ref{thm: gradient error bound}, we then have
\begin{align}
    |\nb \L_t|^2 &\leq 2|\hnb \L_t|^2 + 2|E_t|^2 \nonumber  \\
    &= \O(g^2 + \d^{c'}\cdot \mathrm{poly}(\lmax, g^{-1}, G, L, F, M)). \label{eq: small grad terminate}
\end{align}
}
\end{proof}

We remark that, for a given accuracy level $\d$, we should take a time horizon $T \propto \d^{-1}$. Increasing $T$ beyond this point will not cause the error bound from Theorem \ref{thm: convergence} to decay any further.
%A few remarks are in order. First, in order for inequality (\ref{eq: telescope}) to hold, we must have $\n < L$. Since we treat $L$ as a constant, it suffices to take $\n = \d^c$ for $c > 0$ provided that $\d$ is small enough. Second, as we are taking a step size which decreases with $\d$, we will necessarily have to take more steps for the algorithm to converge, i.e. we should have $T = \omega(\d^{-c})$.
%Finally, the constant $c$ can be chosen freely in $0 < c < 1$. Depending on the relative sizes of $T$, $M$, $G$, and $g$, different choices are more or less appropriate.

\section{Convergence of PerfGD with stochastic errors and general $H$} \label{sec: long horizon convergence}
When the errors on the estimate for $f$ are bounded and deterministic, we gain no advantage by increasing the length of the estimation horizon $H$. However, when the errors are centered and stochastic, the estimation horizon now plays a critical roll. Increasing $H$ allows for concentration of the errors, leading to overall better estimates for $f$. At the same time, increasing $H$ causes the deterministic bias from our finite difference approximations to increase. In the following section, we show how to balance these two factors and choose an optimal $H$. First, we state our main theorem.

\begin{theorem} \label{thm: long horizon perfgd bound}
With step size
%$$\n = \frac{g^{4/13}}{T^{7/13} G^{10/13} M^{3/13} (\tau \sqrt{\log \frac{T}{\g}})^{4/13}}$$
$$\n = \frac{g^{2/3}}{M^{1/2} G^{5/3} \tau^{1/3} (\log \frac{T}{\g})^{1/6} T^{5/6}}$$
and estimation horizon
%$$H = \left(\frac{\tau \sqrt{\log\frac{T}{\g}}}{MGg\n^2}\right)^{2/7},$$
$$H = \frac{\tau^{2/5} (\log \frac{T}{\g})^{1/5}}{M^{2/5} g^{4/5}} \n^{-4/5}$$
the iterates of PerfGD satisfy
%$$\min_{1 \leq t \leq T} |\nb\L_t|^2 = \O\left(\lmax \left[ \frac{1}{T^{6/13}} \left( \frac{\tau \sqrt{\log\frac{T}{\g}}}{g}\right)^{4/13} M^{3/13} G^{10/13} \right] \right)$$
$$\min_{1 \leq t \leq T} |\nb\L_t|^2 = \max\left\{ \O\left(\lmax \left[ \frac{1}{T^{1/6}} \cdot \frac{M^{1/2} G^{5/3} \tau^{1/3} (\log \frac{T}{\g})^{1/6}}{g^{2/3}}\right]\right), \O(g^2 + \tau^{c_1} T^{-c_2}\cdot \mathrm{poly}(M, G, \g^{-1}, g^{-1}) )\right\}$$
with probability at least $1-\O(\g)$ as $\tau \rightarrow 0$ and $T \rightarrow \infty$ and for some positive constants $c_1, c_2$.
\end{theorem}
We remark briefly that we choose to analyze $\tau \rightarrow 0$ since if the estimates for $f_t$ are computed from random samples of increasing size, then we expect the variance of these estimates (measured by $\tau$) to decay to zero as the sample size $n \rightarrow \infty$. For instance, for estimating the mean of a Gaussian we will have $\tau^2 = \O(1/n)$.

The proof of Theorem \ref{thm: long horizon perfgd bound} follows from two key lemmas.

\begin{lemma} \label{thm: subgaussian}
If $X$ is $\tau^2$-subgaussian and $Y$ is any random variable with $|Y| \leq B$ w.p. 1, then $XY$ is $B^2 \tau^2$-subgaussian.
%Furthermore, if $X_i$ is $\tau_i^2$-subgaussian for $i=1,\ldots,H$, then $\sum_{i=1}^H X_i$ is $\O(H\sum_{i=1}^H \tau_i^2)$-subgaussian.
\end{lemma}
A critical fact about this lemma is that the random variables involved need not be independent.

\begin{proof}
By definition, $Z$ is $s^2$-subgaussian if $\E e^{Z^2/s^2} \leq 2$. Observe that since the exponential function is monotonic, we have
$$\E e^{X^2 Y^2 / B^2 \tau^2} \leq \E e^{X^2 B^2 / B^2 \tau^2} = \E e^{X^2/\tau^2} \leq 2.$$
Thus $XY$ is $B^2\tau^2$-subgaussian.

\blockcomment{
For the second part of the lemma, we first assume that $H = 2^h$ and proceed by induction on $h$. The claim is trivially true with $h = 0$; indeed, we have $X_1$ is $2^0 \tau_1^2$-subgaussian. Now assume that $\sum_{i=1}^{2^h}X_i$ is $2^h\sum_{i=1}^{2^h} \tau_i^2$-subgaussian for any collection of $2^h$ subgaussian $X_i$. We use an equivalent characterization of subgaussian random variables: $Z$ is $s^2$-subgaussian if $\E e^{tZ} \leq e^{s^2 t^2 / 2}$. Thus we have
\begin{align*}
    \E \exp(t\sum_{i=1}^{2^{h+1}} X_i) &= \E \exp(t\sum_{i=1}^{2^h} X_i) \exp(t\sum_{i=2^h + 1}^{2^{h+1}} X_i) \\
    &\leq (\E \exp(2t\sum_{i=1}^{2^h} X_i) )^{1/2} (\E \exp(2t\sum_{i=2^h + 1}^{2^{h+1}} X_i))^{1/2} &\textrm{(Cauchy-Schwarz)} \\
    &\leq \exp(4t^2 2^h \sum_{i=1}^{2^h} \tau_i^2 / 2) \exp(4t^2 \sum_{i=2^h + 1}^{2^{h+1}} 2^h \tau_i^2 /2) &\textrm{(Inductive hypothesis)}\\
    &= \exp(t^2 2^{h+1} \sum_{i=1}^{2^{h+1}} \tau_i^2 / 2).
\end{align*}
Thus $\sum_{i=1}^{2^{h+1}} X_i$ is $2^{h+1}\sum_{i=1}^{2^{h+1}} \tau_i^2$-subgaussian and the result follows by induction. To get the result for general $H$, observe that $H$ can be brought to a power of 2 by adding at most $H$ 0-subgaussian random variables (which are just 0). Then the factor $2^h$ is at most $2H$, so we have $\sum_{i=1}^H X_i$ is $\O(H\sum_{i=1}^H \tau_i^2)$-subgaussian as desired.
}
\end{proof}

\begin{lemma} \label{thm: long horizon f' error bound}
We have $\hfpt = f_t' + b_t + e_t$, where $b_t$ is a deterministic bias term with $|b_t| = \O(MGH\n)$. Under the additional assumption that $\th_t$ converge monotonically, $e_t$ is $\O\left(\frac{G^2\tau^2}{H^3g^4\n^2}\right)$-subgaussian.
\end{lemma}

\begin{proof}
The pseudoinverse used to compute $\hfpt$ is equivalent to solving the least-squares problem
\begin{equation} \label{eq: least sq for hfpt}
    \hfpt = \argmin_{\alpha} \frac12 \sum_{i=1}^H (\alpha (\th_{t-i} - \th_t) - (\hat{f}_{t-i} - \hft))^2 \hspace{.1in} \Longrightarrow \hspace{.1in} \hfpt = \frac{\sum_{i=1}^H (\hat{f}_{t-i} - \hft)(\th_{t-i} - \th_t)}{\sum_{i=1}^H (\th_{t-i} - \th_t)^2}.
\end{equation}
Writing $\hft = f_t + \e_t$ with $\e_t$ $\tau$-subgaussian, we can apply Taylor's theorem to rewrite
\begin{equation} \label{eq: rewrite Delta f}
    \hat{f}_{t-i} - \hft = f'_t(\th_{t-i} - \th_t) + \frac12 f''(\xi_i)(\th_{t-i} - \th_t)^2 + \e_{t-i} - \e_t.
\end{equation}
Using the explicit solution in (\ref{eq: least sq for hfpt}) and substituting (\ref{eq: rewrite Delta f}) for $\hat{f}_{t-i} - \hft$, we find that
\begin{equation*}
    |\hfpt - f'_t| \leq \underbrace{\frac{\frac12\sum_{i=1}^H |f''(\xi_i)||\th_{t-i} - \th_t|^3}{\sum_{i=1}^H (\th_{t-i} - \th_t)^2}}_{b_t} + \underbrace{\frac{\sum_{i=1}^H (\e_{t-i} - \e_t)(\th_{t-i} - \th_t)}{\sum_{i=1}^H (\th_{t-i} - \th_t)^2}}_{e_t}.
\end{equation*}
To bound $b_t$, observe that since $\th_t = \th_{t-i} - \n (\hnb \L_{t-i} + \cdots + \hnb \L_{t-1})$ and $|\hnb \L_s| \leq G$ and $i \leq H$, we have $|\th_{t-i} - \th_t| \leq HG\n$ for all $i, t$. Since $|f''(\xi_i)| \leq M$, we have
\begin{align*}
    |b_t| &\leq \frac{ \frac12 \sum_{i=1}^H MHG\n (\th_{t-i} - \th_t)^2}{\sum_{i=1}^H (\th_{t-i} - \th_t)^2} \\[10pt]
    &= \O(MGH\n).
\end{align*}
Next we bound $e_t$. Since we have assumed that $\th_t$ converge monotonically and $|\hnb \L_t|\geq g$, we have
$$\frac{1}{\sum_{i=1}^H (\th_{t-i} - \th_t)^2} \leq \frac{1}{\sum_{i=1}^H (ig\n)^2} = \O(\frac{1}{H^3 g^2\n^2}).$$
In the numerator, we have
\begin{equation*}
|\sum_{i=1}^H (\e_{t-i} - \e_t)(\th_{t-i} - \th_t)| \leq HG\n \sum_{i=1}^H |\e_{t-i} - \e_t|.
\end{equation*}
Combining these, we have
\begin{equation}\label{eq: bounding et}e_t = \O\left(\frac{G}{H^2 g^2 \n}\right)\sum_{i=1}^H |\e_{t-i} - \e_t|.\end{equation}
We make the additional simplifying assumption that the $|\e_{t-i} - \e_t|$ are independent. We can accomplish this splitting our dataset drawn from $\D(\th_t)$ into $H$ parts and estimating $f_t$ once with each component, then replacing the terms $(\hat{f}_{t-i} - \hft)$ with $(\hat{f}_{t-i} - \hat{f}_{t, i})$ in equation (\ref{eq: least sq for hfpt}), where $\hat{f}_{t, i}$ is the estimate of $f_t$ from the $i$-th partition of the dataset. The errors $\e_t$ in the expression for $e_t$ now become independent copies $\e_{t, i}$, and the terms in equation (\ref{eq: bounding et}) are indeed independent.

Under this assumption, $|\e_{t-i} - \e_{t, i}|$ are independent $2\tau^2$-subgaussian random variables. Their sum is therefore $\sum_{i=1}^H 2\tau^2 = \O(H\tau^2)$-subgaussian. Finally, by Lemma \ref{thm: subgaussian}, it follows that $e_t$ is $\O(\frac{G^2\tau^2}{H^3g^4\n^2})$-subgaussian.

\end{proof}

With these two lemmas, we can now prove the main theorem. The structure of the proof is similar to that of Theorem \ref{thm: convergence}.
\begin{proof}[Proof of Theorem \ref{thm: long horizon perfgd bound}]
We first establish a high-probability bound on $|e_t|$. By the subgaussian tail bound and a union bound over $t = 1$ to $T$, a simple calculation shows that $$|e_t| = \O\left(\frac{G\tau\sqrt{\log \frac{T}{\g}}}{g^2 \n H^{3/2}}\right)$$ with probability at least $1-\g$ for all $t=1,\ldots,T$. Combining this bound with the bound on $|b_t|$ from Lemma \ref{thm: long horizon f' error bound}, we find that
\begin{equation*}
    |\hfpt - f'_t| = \O\left(MG\n H + \frac{ G \tau \sqrt{\log \frac{T}{\g}} }{g^2\n} H^{-3/2}\right).
\end{equation*}
With $H$ chosen as is in the theorem, this bound simplifies to
\begin{equation} \label{eq: high prob f' error bound}
    |\hfpt - f'_t| = \O\left( \frac{M^{3/5} G \tau^{2/5} (\log \frac{T}{\g})^{1/5}}{g^{4/5}} \n^{1/5} \right) \equiv \mathbf{E}_1.
\end{equation}
From the proof of Lemma \ref{thm: gradient error bound}, we know that
\begin{equation} \label{eq: long horizon grad bound 1}
    |\hnb \L_t - \nb \L_t| = \O\left(\lmax \left[ \mathbf{E}_1 + F\d \sqrt{\log \frac1\d} \right] \right),
\end{equation}
where $\d$ is a (high-probability) bound on the error of $f_t$. Again assuming that this error is $\tau^2$-subgaussian, we have that
$$(\textrm{error on } f_t) = \O\left(\tau\sqrt{\log\frac{T}{\g}}\right)$$
for all $t=1,\ldots,T$ with probability at least $1-\g$. Thus we can take $\d = \tau \sqrt{\log(T/\g)}$, in which case the second term in equation (\ref{eq: long horizon grad bound 1}) is $\O(\mathbf{E}_1)$ as $\tau \downarrow 0$. It follows that $|\hnb \L_t - \nb \L_t| = \O(\lmax \mathbf{E}_1) \equiv \mathbf{E}_2$ with high probability.

Finally, by the same analysis used in the proof of Theorem \ref{thm: convergence}, we have that
$$ \min_{1\leq t \leq T} |\nb \L_t|^2 = \O\left(\frac{\lmax + \n G T \mathbf{E}_2}{T\n}\right). $$
Choosing $\n$ as in the theorem statement and substituting our bound on $\mathbf{E}_2$ yields the desired result. The max in the theorem statement follows from the same logic as in Theorem \ref{thm: convergence} plus the bound on the error performative gradient error $\mathbf{E}_2$.
\end{proof}

\section{Experiment details}
In all of the following experiments, whenever the stated estimation horizon $H$ is longer than the entire history on a particular iteration of PerfGD, we simply use $H = $ length of the existing history for that iteration instead. Furthermore, in all of the experiments, both RGD and PerfGD were run using a learning rate of $\n = 0.1$.

\subsection{Mixture of Gaussians and nonlinear mean (\S \ref{sec: toy})}
For the nonlinear mean experiment, we set $a_0 = a_1 = 1$ and $\s^2 = 1$. At each iteration, we drew $n=500$ data points. We initialized PerfGD using only one step of RGD, and at each step after the initialization we used the previous $H = 4$ steps to estimate $\mu'(\th)$. The analytical values for $\opt$ and $\stab$ are given by
$$\opt = -\frac{2a_0}{3a_1}, \hspace{.25in} \stab = -\frac{a_0}{a_1}.$$

For the Gaussian mixture experiment, we set $\g = 0.5$, $\s_1^2 = 1$, $a_{1,0} = -0.5$, $a_{1,1} = 1$, $s_2^2 = 0.25$, $a_{2,0} = 1$, and $a_{2,1} = -0.3$. At each iteration, we drew $n = 1000$ data points. We initialized PerfGD using only one step of RGD, and at each step after the initialization we use the entire history to estimate $\mu_i'(\th)$. The analytical values for $\opt$ and $\stab$ are given by
$$\opt = -\frac12\frac{\g a_{1, 0} + (1-\g) a_{2, 0}}{\g a_{1, 1} + (1-\g) a_{2, 1}}, \hspace{.25in} \stab = \frac{\g a_{1, 0} + (1-\g) a_{2, 0}}{\g a_{1, 1} + (1-\g) a_{2, 1}}.$$

\subsection{Pricing (\S \ref{sec: pricing})}
We set $d = 5$ for this experiment. We then set $\mu_0 = 6 \cdot \mathbf{1} + \mathrm{Unif}[0, 1]^5$ with a fixed random seed; in this case, it came out to $\mu_0 \approx [6.55, 6.72, 6.60, 6.54, 6.42]^\T$. We set $\Sigma = I \in \R^{5\times 5}$ (i.e. the $5\times 5$ identity matrix) and $\e = 1.5$. At each iteration, we drew $n = 500$ data points. We initialized PerfGD with 14 steps of RGD, and at each step after initialization we used the entire history to estimate $d\mu/d\th$. The analytical values for $\opt$ and $\stab$ are given by
$$\opt = \frac{\mu_0}{2\e}, \hspace{.25in} \stab = \frac{\mu_0}{\e}.$$

\subsection{Binary classification (\S \ref{sec: classification})}
Here the features $x\in \R$ are one-dimensional, while our model parameters $\th \in \R^2$ allow for a bias term. We set $\s_0^2 = 0.25$, $\mu_0 = 1$, $\s_1^2 = 0.25$, $\mu_1 = -1$, and $\e = 3$. The regularization strength for $\ell$ was $\lambda = 10^{-2}$, i.e.
$$\ell(x, y; \th) = -y \log h_\th(x) - (1-y) \log (1 - h_\th(x)) + \frac{10^{-2}}{2} \lVert \th \rVert^2.$$
When approximating the derivatives of the means of the mixtures with respect to $\th$, we assume that it is known that the derivative of the non-spam email mean is independent of $\th$, and we also assume knowledge of the fact that the mean of the spam email features depends only on $\th_1$ (i.e. the non-bias parameter). At each iteration, we drew $n=500$ data points. We initialize PerfGD using only one step of RGD, and at each step after the initialization we use the entire history to estimate $f'(\th)$.

\subsection{Regression (\S \ref{sec: regression})}
We set $\mu_x = 1.67$, $\s_x^2 = 1$, $a_0 = a_1 = 1.67$, and regularization strength $\lambda = 3.33$ for the loss, i.e.
$$\ell(x, y; \th) = \frac12 (\th x - y)^2 + \frac{3.33}{2} |\th|^2.$$
The variance of $y|x$ was set to $4.12$. At each iteration, we drew $n = 500$ data points. The analytical values for $\opt$ and $\stab$ are given by
$$\opt = \frac{c \cdot a_0}{ c \cdot (1-a_1) + \frac{\lambda}{1-a_1} }, \hspace{.25in} \stab = \frac{c \cdot a_0}{ c \cdot (1-a_1) + \lambda },$$
where $c = \mu_x^2 + \s_x^2$.

\end{document}